\documentclass[runningheads,fleqn]{llncs}
\usepackage[T1]{fontenc}
\usepackage{graphicx}
\usepackage{hyperref}

\usepackage{color}

\usepackage{amsmath,amssymb,verbatim,cite}
\usepackage{upgreek}
\hypersetup{breaklinks=true}

\usepackage{cleveref}

\usepackage[inline]{enumitem}
 
\newcommand{\cR}{{\mathcal R}}
\def\cF{\mathcal{F}}
\def\cV{\mathcal{V}}
\def\cT{\mathcal{T}}

\newcommand{\llbrace}{\{\!\!\{}
\newcommand{\rrbrace}{\}\!\!\}}
\newcommand\mset[1]{\llbrace{#1}\rrbrace}

\def\size{\mathit{size}}

\def\cm{\mathit{cm}}

\def\st{\uptau}
\def\sr{\uprho}
\def\sp{\upphi}

\let\ra\rightarrow

\def\T{\mathcal{T}}

\def\fv{\mathtt{fv}}
\def\bv{\mathtt{bv}}
\def\cut{\mu}

\def\deg{\mathfrak{d}}

\def\A{t}
\def\B{s}
\def\C{r}

\newcommand{\dom}{\mathit{Dom}}
\newcommand{\ran}{\mathit{Ran}}

\newcommand{\pos}{\mathit{Pos}}

\newcommand{{\hopsu}}{{\textsc{Hops}}}
\newcommand{{\stepsu}}{{\textsf{Step}}}
\newcommand{{\hoppu}}{{\textsc{Hop-PU}}}
\newcommand{{\steppu}}{{\textsf{Step-PU}}}
\newcommand{\varelimsu}{\mathsf{VarElim}}

\makeatletter
\def\infrule#1#2#3{\@ifnextchar[{\@infrule{#1}{#2}{#3}}{\@infrule{#1}{#2}{#3}[*]}}%

\def\@infrule#1#2#3[#4]{
\par\bigbreak
\vtop{
\hangindent3em\hangafter2\leavevmode\null
\textsf{#1:}\kern.5em\textbf{#2}\\[\smallskipamount]
\ifx*#4 
\null\qquad$#3$
\else 
\setbox30=\hbox{\qquad$#3$\qquad#4}%
\ifdim\wd30>\hsize
 \null\qquad$#3$\par\kern-\parskip\smallskip#4
\else
 \null\qquad$#3$\qquad#4
\fi
\fi
}%
} \makeatother

\allowdisplaybreaks

\begin{document}
\title{Higher-Order Pattern Unification Modulo Similarity Relations}
%
%
\author{Besik Dundua\inst{1,2}
\and Temur Kutsia\inst{3}
}
\authorrunning{B. Dundua and T. Kutsia}
%
%
\institute{
VIAM, Tbilisi State University, Georgia 
\and Kutaisi International University, Georgia\\ 
\email{bdundua@gmail.com}
\and 
RISC, Johannes Kepler University, Linz, Austria\\
\email{kutsia@risc.jku.at}
}
\maketitle              
%

\begin{abstract}
The combination of higher-order theories and fuzzy logic can be useful in decision-making tasks that involve reasoning across abstract functions and predicates, where exact matches are often rare or unnecessary. Developing efficient reasoning and computational techniques for such a combined formalism presents a significant challenge. In this paper, we adopt a more straightforward approach aiming at integrating two well-established and computationally well-behaved components: higher-order patterns on one side and fuzzy equivalences expressed through similarity relations based on minimum T-norm on the other. We propose a unification algorithm for higher-order patterns modulo these similarity relations and prove its termination, soundness, and completeness. This unification problem, like its crisp counterpart, is unitary. The algorithm computes a most general unifier with the highest degree of approximation when the given terms are unifiable.

\end{abstract}

\section{Introduction}
 \label{sect:intro}

Approximate reasoning involves making decisions or performing inferences based on vague or imprecise information, which is often modeled using fuzzy logic. Fuzzy similarity (and proximity) relations are key tools in this type of reasoning, allowing for working with uncertain or imprecise data. Fuzzy similarity refers to the degree to which two elements or objects are alike, using a value in the range of 0 to 1, where 0 represents no similarity and 1 represents complete similarity. These relations can be used to determine how similar a new situation is to past experiences or data. This is particularly useful in scenarios where precise comparisons are impossible or where ``similar enough'' is acceptable. (See, e.g.,~\cite{DBLP:journals/tfs/MilaneseP24,DBLP:journals/tcs/Sessa02,RISC6513,DBLP:conf/fqas/KrajciLMOV02,DBLP:journals/ijis/FontanaF02,DBLP:journals/fss/IranzoMR23,DBLP:conf/fuzzIEEE/SandriMMGC12,DBLP:journals/fss/LoiaSS04} in the context of using similarity relations in automated approximate reasoning.) 

Fuzzy similarity and proximity relations have been incorporated into knowledge representation and inference processes in the area of logic programming~\cite{DBLP:conf/agp/FormatoGS99,DBLP:journals/tcs/Sessa02,DBLP:journals/fss/IranzoR17,DBLP:conf/sac/ArcelliF99,DBLP:journals/ijar/IranzoMR20,DBLP:journals/eswa/IranzoS23,DBLP:conf/ssci/GuerreroMRS18}, making it possible to reason or compute with vague concepts having imprecise boundaries. It helps to deal with the ambiguity inherent in real-world data, where exact matches are rare, and allows systems to make reasonable inferences even with imprecise or incomplete information. In order for this approach to work, one needs a fundamental computational mechanism such as unification to make an ``approximate inference step''. Motivated by this application, first-order unification with fuzzy relations has been investigated both from theoretical and practical points of view, see, e.g., \cite{DBLP:conf/fuzzIEEE/PauK21,DBLP:journals/tcs/Sessa02,DBLP:conf/fqas/KrajciLMOV02,DBLP:journals/fss/MedinaOV04,DBLP:journals/fss/Ait-KaciP20,DBLP:journals/fuin/FormatoGS00,DBLP:journals/ijis/FontanaF02,DBLP:journals/entcs/Virtanen02,DBLP:journals/kybernetika/Vojtas00,DBLP:conf/lopstr/KutsiaP19,DBLP:conf/fuzzIEEE/IranzoS18,DBLP:journals/fss/IranzoR15,RISC6513,DBLP:conf/fscd/DunduaKMP20}. In \cite{DBLP:conf/ijcar/EhlingK24}, this technique was studied in a more general setting where the quantitative approximate information is specified using quantales, having the fuzzy quantale as a special case. 

The cited works focus on fuzzy reasoning and computation within a first-order framework. However, decision-making often requires higher-order reasoning across abstract layers of functions and predicates, where exact matches are rare or even unnecessary. For instance, when two functions are not identical but exhibit a high degree of similarity in behavior or output, fuzzy similarity relations would enable the system to recognize this resemblance and apply approximate reasoning for inference or prediction. This can be especially useful in fields such as natural language processing, knowledge representation, and complex decision-making systems, where human-like reasoning and the handling of uncertainty are essential. This idea is a primary motivation for our work.

It is important to note that the existing higher-order fuzzy logics (e.g., \cite{DBLP:conf/fuzzIEEE/Maruyama21, DBLP:journals/fss/Novak12a}) are highly expressive formalisms, but because of this expressive power, developing efficient reasoning and computational techniques for them is a significant challenge. We adopt a more straightforward approach, concentrating on the integration of established and computationally well-behaved fragments: higher-order patterns on one hand, and fuzzy equivalences expressed through minimum-T-norm-based similarity relations on the other. The first step in supporting reasoning for this integration is the development of fundamental computational techniques, such as unification, which is the subject of this paper.

\paragraph*{Our contribution.} We define the notion of similarity for simply-typed lambda terms and study the unification problem for higher-order patterns \`a la Miller~\cite{DBLP:journals/logcom/Miller91}. In this framework, the equality relation modulo $\alpha\beta\eta$-equivalence is replaced by fuzzy similarity modulo $\alpha\beta\eta$ with the minimum T-norm (G\"odel T-norm). We develop a rule-based unification algorithm for such problems and prove its termination, soundness, and completeness. The unification problem, like its crisp counterpart, is unitary, that is, any unifiable problem has a \emph{single} (modulo $\alpha\beta\eta$-equivalence) \emph{most general unifier} (mgu) with respect to fuzzy similarity using minimum T-norm. The algorithm returns such an mgu for the given terms, along with an associated approximation degree, which indicates how similar the instances of the terms are under the unifier. Importantly, the computed unifier has the maximal possible approximation degree. This degree must meet or exceed a user-defined threshold (also known as a cut value). When the threshold is set to the maximal value $1$, the algorithm computes standard (i.e., crisp) unifiers. In this case, it can be viewed as a rule-based version of the standard higher-order pattern unification algorithm.

Our presentation differs from traditional accounts of higher-order pattern unification and fuzzy unification by introducing a dedicated subalgorithm for variable elimination. This design has two main benefits. First, it simplifies the termination proof. Second, it enhances modularity, making it easier to integrate alternative T-norms: only the subalgorithm requires modification in such cases.

\paragraph{Organization.} Section~\ref{sect:prelim} introduces the basic notions and establishes fundamental properties. The core of the paper is Section~\ref{sect:sim:unif}, where we define the algorithm and prove its termination, soundness, and completeness. In Section~\ref{sect:discussion}, we discuss two topics: (a) the crisp counterpart of our algorithm and its relation to existing higher-order pattern unification algorithms; and (b) possible extensions using T-norms other than the minimum (G\"odel) T-norm employed in this work. Section~\ref{sect:concl} concludes the paper. 

\section{Preliminaries}
\label{sect:prelim}
\subsubsection*{Term Language.}

Given a set of basic types, whose elements are denoted by $\updelta$, \emph{simple types} are constructed  using the grammar $\uptau ::= \updelta \mid \uptau \to \uptau$, where $\uptau$ is associative to the right. The alphabet of our language consists of the set $\cV$ of variables and $\cF$ of constants. They are disjoint and countably infinite and their elements have assigned types. It is assumed that for each basic type, there is at least one function symbol (types are not empty). Variables are typically denoted by $F,G,H, X, Y, Z , x, y, z, \ldots$ and constants by $f, g,a,b,c,\ldots$. The set of \emph{terms} over $\cF$ and $\cV$, denoted by $\T(\cF,\cV)$, is defined by the grammar
\[\A ::= x \mid c \mid \lambda x.\A \mid (\A_1\,\A_2),\]
where $\lambda x. \A$ is called an \emph{abstraction} and $(\A_1 \, \A_2)$ is called an \emph{application}. We denote terms by $\A,\B,\C$. 

The following standard abbreviations are used: $(\A_1\, \A_2\, \A_3\, \cdots\, \A_n)$, $n\ge 3$, for $(\cdots ((\A_1\,\A_2)\,\A_3)\, \cdots \A_n)$, and $\lambda x_1,\ldots,x_n.\A$, $n\ge 2$, for $\lambda x_1. (\lambda x_2. (\ldots (\lambda x_n. \A)))$. 

A term $\A$ is said to \emph{have the type} $\st$ if either
\begin{itemize}
\item $\A$ is a constant or a variable of type $\st$,
\item $\A = \lambda x. \B$, the variable $x$ has type $\sr$, the term $\B$ has type $\sp$, and $\st = \sr \ra \sp$,
\item $\A=(\B\,\C)$, where the term $\B$ has type $\sr \ra \st$ and the term $\C$ has type $\sr$ for some $\sr$.
\end{itemize}

The standard concepts of the simply-typed $\lambda$-calculus, such as bound and free occurrences of variables, position in a term, $\alpha$-conversion, $\beta$-reduction, and $\eta$-long $\beta$-normal form, are defined in the usual way, see, e.g., \cite{DBLP:books/cp/BarendregtM22, DBLP:books/el/RV01/Dowek01}. The set of positions of a term $\A$ is denoted by $\pos(\A)$. The $\eta$-long $\beta$-normal form of a term $\A$ is denoted by $\A\updownarrow_\beta^\eta$. 
For any $\A$, the term $\A\updownarrow_\beta^\eta$ has the form $\lambda x_1, \dots x_n. (h\,t_1\, \cdots\, t_m)$, where $n, m \geq 0$, the symbol $h$ (called the \emph{head} of the term) is either a constant or a variable, and each $t_i$ (for $i = 1, \dots, m$) follows the same structural form. Moreover, the term $(h\,t_1\, \cdots\, t_m)$ has a basic type. We follow the standard convention of writing terms in $\eta$-long $\beta$-normal form as $\lambda x_1, \dots, x_n.h(t_1, \dots, t_m)$. When we write an equality between two $\lambda$-terms, we mean that they are equivalent modulo $\alpha$-, $\beta$-, and $\eta$-equivalence. The size of a term $\A$, denoted $\size(\A)$, is defined recursively as $\size(x)= \size(f)=1$, $\size((\A_1\, \A_2)) = \size(\A_1)+\size(\A_2)$, and $\size(\lambda x.\A) = 1 + \size(\A)$. 

The sets of free and bound variables of a term $\A$ are denoted by $\fv(\A)$ and $\bv(\A)$, respectively. As a convention and for the sake of clarity, in what follows, we distinguish bound and free variables syntactically. In particular, we use lowercase letters $x, y, z$ for bound variables and capital letters $X, Y, Z, F, G, H$ for free variables. A term is called \emph{rigid} if its head symbol is a constant or a bound variable, and it is called \emph{flexible}, if its head symbol is a free variable.

\begin{definition}[Higher-order patterns]
A higher-order pattern is a term where, when written in $\eta$-long $\beta$-normal form, all
free variable occurrences are applied to lists of pairwise distinct ($\eta$-long forms of) bound
variables.
\end{definition}

A \emph{substitution} $\sigma$ is a mapping from variables to terms such that for each variable $x$ of type $\st$, the term $\sigma(x)$ is of type $\st$, and all but finitely many variables are mapped to themselves (modulo $\alpha\beta\eta$). Greek letters $\sigma,\vartheta,\varphi,\tau,\varepsilon$ are used for substitutions, where $\varepsilon$ denotes the identity substitution. Each substitution $\sigma $ is represented as a finite set of pairs $\{x_1\mapsto \sigma(x_1),\ldots,$ $ x_n\mapsto \sigma(x_n)\}$ where the $x$'s are variables for which $\sigma(x_i) \neq x_i$. The sets $\dom(\sigma)=\{x_1,\ldots,x_n\}$ and $\ran(\sigma)=\{\sigma(x_1),\ldots, \sigma(x_n)\}$ are called the \emph{domain} and the \emph{range} of $\sigma $, respectively.
The set $\fv(\sigma)$ is defined as $\fv(\sigma)=\dom(\sigma)\cup \fv(\ran(\sigma))$ where $\fv(\ran(\sigma))=  \cup_{i=1}^n \fv(\sigma(x_i))$

The \emph{application} of a substitution $\sigma $ to $\A$ replaces each \emph{free} occurrence of a variable $x$ in $\A$ with $\sigma(x)$. It is defined inductively: 
\begin{alignat*}{7}
&x \sigma       =  \sigma(x), &\qquad  & \text{ if $x\in\dom(\sigma)$.} \\
&x \sigma       =  x, & & \text{ if $x\notin\dom(\sigma)$.} \\ 
&f \sigma       =  f. & \\
& (\lambda x.\A) \sigma = \lambda x  .\A \sigma,  & & \text{ if $x\notin\fv(\sigma)$.} \\
&(\lambda x.\A) \sigma = \lambda y  .\A\{x\mapsto y\} \sigma,  & & \text{ if $x\in\fv(\sigma)$ and $y$ is fresh.} \\
& (\A\, \B) \sigma    = (\A \sigma \,\B \sigma). 
\end{alignat*}

\subsubsection*{Fuzzy Relations.}

We define basic notions about fuzzy relations following \cite{DBLP:journals/tcs/Sessa02}.
A binary \emph{fuzzy relation} on a set $S$ is a mapping from $S\times S$ to the real interval $[0,1]$. If $\cR$ is a fuzzy relation on $S$ and $\cut$ is a number $0<\cut \le 1$ (called \emph{cut value}), then the \emph{$\cut$-cut} of $\cR$ on $S$, denoted $\cR_\cut$, is an ordinary (crisp) relation on $S$ defined as
$\cR_\cut := \{(s_1,s_2) \mid \cR(s_1,s_2) \ge \cut \}$. 

A \emph{T-norm} $\wedge$ is an associative, commutative, non-decreasing binary operation on $[0,1]$ with 1 as the unit element. Some of the most prominent T-norms are

\begin{itemize}
    \item G\"odel (or minimum) T-norm: $s_1\wedge s_2=\min(s_1,s_2)$,
    \item Product T-norm: $s_1\wedge s_2=s_1 * s_2$,
    \item {\L}ukasiewicz T-norm: $s_1\wedge s_2=\max(0, s_1+s_2-1)$.
\end{itemize}

\begin{definition}[Similarity relation]
A fuzzy relation $\cR$ on a set $S$ is called a \emph{proximity relation}, if it is reflexive ($\cR(s,s)=1$ for all $s\in S$) and symmetric ($\cR(s_1,s_2)=\cR(s_2,s_1)$ for all $s_1,s_2\in S$).  A proximity relation is called a \emph{similarity relation} if it is $\wedge$-transitive: $\cR(s_1,s_2)\ge \cR(s_1,s) \wedge \cR(s,s_2)$ for any $s_1,s_2,s\in S$.
\end{definition}

In the role of $S$, we take the set of terms of our language. First, we assume a fuzzy relation $\cR_A$ to be defined on the alphabet $\cF\cup \cV$ in such a way that

\begin{itemize}
    \item $\cR_A(x,y)=0$ for all $x,y\in \cV$ with $x\neq y$,
    \item $\cR_A(f,g)=0$ for all $f,g\in \cF$ such that $f$ and $g$ have different types. 
    \item $\cR_A(x,f)=\cR_A(f,x)=0$ for all $x\in\cV$ and $f\in \cF$.
\end{itemize}

\begin{definition}[Fuzzy relation on terms]
Given a fuzzy relation $\cR_A$ on the alphabet $\cF\cup \cV$, we define a fuzzy relation $\cR$ on the set of terms $\cT(\cF,\cV)$ using the T-norm $\wedge$.  

\begin{enumerate}
    \item If $\A$ and $\B$ are in $\eta$-long $\beta$-normal form, then $\cR(\A,\B)$ is defined as follows:
\begin{align*}
    &\cR(a,b)=\cR_A(a,b), \text{ where $a,b \in \cF \cup \cV$,} \\
    &\cR((\A_1\, \B_1),(\A_2\, \B_2))=  \cR(\A_1,\A_2) \land \cR(\B_1,\B_2),\\
    &\cR(\lambda x. \A, \lambda y. \B) = \cR(\A\{x\mapsto z\}, \B\{y\mapsto z\}), \\
    & \qquad \text{where $x$, $y$, and $z$ have the same type and $z$ is a fresh variable,}\\
    & \cR(\A, \B) = 0 \text{ otherwise.}
\end{align*}
\item Otherwise, $\cR(\A,\B) = \cR(\A\updownarrow_\beta^\eta,\B\updownarrow_\beta^\eta)$.
\end{enumerate}
\end{definition}

It is easy to see that if $\cR_A$ is a proximity relation on the alphabet, then $\cR$ is a proximity relation on terms for any T-norm. This definition also implies that if $\cR(\A,\B)>0$, then 

\begin{itemize}
    \item $\A$ and $\B$ have the same type,
    \item $\pos(\A\updownarrow_\beta^\eta)=\pos(\B\updownarrow_\beta^\eta)$,
    \item for each $p\in \pos(\A\updownarrow_\beta^\eta)$, symbols occurring at position $p$ in $\A\updownarrow_\beta^\eta$ and $\B\updownarrow_\beta^\eta$ have the same type,
    \item a variable (resp. a constant) occurs at position $p$ in $\A\updownarrow_\beta^\eta$ iff a variable (resp. a constant) occurs at position $p$ in $\B\updownarrow_\beta^\eta$,
    \item a free variable $X$ occurs at position $p$ in $\A\updownarrow_\beta^\eta$ iff the same variable $X$ occurs at position $p$ in $\B\updownarrow_\beta^\eta$.
\end{itemize}

Note that if the T-norm is not idempotent, one can find a substitution $\sigma$ such that $\cR(\A,\B)>0$ does not imply $\cR(\A,\B)=\cR(\A\sigma,\B\sigma)$. For instance, for the product T-norm, $\cR(a,b)=0.5$, and $\sigma=\{X\mapsto \lambda x.f(x,x)\}$ we have
\begin{align*}
   & \cR(X(a), X(b)) = 0.5, \\
   & \cR(X(a)\sigma, X(b)\sigma) = \cR(f(a,a), f(b,b))=0.5 * 0.5  = 0.25.
\end{align*}

However, this problem does not arise when the T-norm is idempotent, or when terms are higher-order patterns: 

\begin{proposition} Let $\cR$ be a similarity relation on terms using the T-norm $\land$, $\A$ and $\B$ be terms, and $\sigma$ be a substitution. Then $\cR(\A,\B)>0$ implies $\cR(\A,\B)=\cR(\A\sigma,\B\sigma)$ if
\begin{enumerate*}[label=(\alph*)]
    \item $\land$ is idempotent, or
    \item $\A$ and $\B$ are higher-order patterns. 
\end{enumerate*}
\end{proposition}
\begin{proof}
    Assume without loss of generality that $\A$ and $\B$ are in $\eta$-long $\beta$-normal form. They have the same structure, since $\cR(\A,\B)>0$. We proceed by structural induction.

    $\A$ and $\B$ are variables. Then $\A=\B$, $\A\sigma=\B\sigma$, and $\cR(\A\sigma,\B\sigma)=\cR(\A,\B)=1$.

    $\A$ and $\B$ are constants. Then $\A\sigma=\A$, $\B\sigma=\B$, and $\cR(\A\sigma,\B\sigma)=\cR(\A,\B)$. 

    $\A = \lambda x_1,\ldots,x_n.h_1(\A'_1,\ldots,\A'_m)$ and $\B = \lambda x_1,\ldots,x_n.h_2(\B'_1,\ldots,\B'_m)$, where $h_1$ and $h_2$ are either both variables or both constants. Then $\cR(\A,\B)= \cR_A(h_1,h_2) \land \cR(\A'_1,\B'_1)\land \cdots \land \cR(\A'_m,\B'_m)$. By the induction hypothesis, $\cR(\A'_i,\B'_i)= \cR(\A'_i\sigma,\B'_i\sigma)$ for all $1\le i \le m$. 
    \begin{itemize}
        \item If $h_1$ and $h_2$ are constants or bound variables, then $h_1\sigma=h_1$, $h_2\sigma=h_2$ and $\cR(\A\sigma,\B\sigma)= \cR(\A,\B)$. 
        \item If $h_1$ and $h_2$ are free variables, then they should be the same (since $\cR(\A,\B)>0$), say $X$. We consider two cases:
        \begin{itemize}
            \item $\A$ and $\B$ are patterns. Then $\A'_i=\B'_i \in \{x_1,\ldots,x_n\}$. Therefore, $\A=\B$ and $\cR(\A\sigma, \B\sigma)=\cR(\A, \B)=1$. 
            \item $\A$ and $\B$ are not patterns. If $X\notin \dom(\sigma)$, then the reasoning is like for bound variables above. If $X\in \dom(\sigma)$, then let $X\sigma = \lambda y_1,\ldots,y_m.\C$. Then 
    \begin{align*}
        \A\sigma ={} & \lambda x_1,\ldots,x_n.r\{y_1\mapsto \A'_1\}\cdots \{y_m \mapsto \A'_m\}, \\
        \B\sigma ={} & \lambda x_1,\ldots,x_n.r\{y_1\mapsto \B'_1\}\cdots \{y_m \mapsto \B'_m\}.
    \end{align*}
           If $r$ is linear (i.e., no $y_i$ occurs in it more than once), then $\cR(\A\sigma, \B\sigma) = \cR(\A'_1,\B'_1)\land \cdots \land \cR(\A'_m,\B'_m) = \cR(\A,\B)$. Now let $\land$ be idempotent and assume without loss of generality that $y_1$ appears in $r$ twice. Then $\cR(\A\sigma, \B\sigma) = \cR(\A'_1,\B'_1)\land \cR(\A'_1,\B'_1)\land \cdots \land \cR(\A'_m,\B'_m) = {}$ (by idempotence of $\land$) ${} = \cR(\A'_1,\B'_1)\land \cdots \land \cR(\A'_m,\B'_m) = \cR(\A,\B)$. \qed
        \end{itemize}
    \end{itemize}
\end{proof}

\begin{theorem}
    If $\cR_A$ is a similarity relation on the alphabet, then $\cR$ is a similarity relation on terms for any T-norm.
\end{theorem}    
\begin{proof} 
    We only need to show that for arbitrary terms $\A, \B,\C$ in $\eta$-long $\beta$-normal form and a T-norm $\wedge$, the $\wedge$-transitivity property holds: $\cR(\A,\C)\ge \cR(\A,\B)\land \cR(\B,\C)$. Assume without loss of generality that $\cR(\A,\B)\land \cR(\B,\C) >0$. Then by the monotonicity of T-norms we have $\cR(\A,\B)>0 $ and $\cR(\B,\C) >0$, which implies that $\A, \B$, and $\C$ have the same structure. We proceed by induction on this structure. 

\begin{itemize}
    \item When $\{\A, \B,\C\}\in \cF \cup \cV$, the property follows from the assumption that $\cR_A$ is a similarity.
    \item Let $\A=(\A_1\,\A_2)$, $\B=(\B_1\,\B_2)$, and $\C=(\C_1\,\C_2)$. 
    By definition of $\cR$, we have $\cR(\A,\C)= \cR(\A_1,\C_1) \land \cR(\A_2,\C_2)$, $\cR(\A,\B) = \cR(\A_1,\B_1) \land \cR(\A_2,\B_2)$, and $\cR(\B,\C) = \cR(\B_1,\C_1) \land \cR(\B_2,\C_2)$. By the induction hypothesis, we know  $\cR(\A_i,\C_i)\ge \cR(\A_i,\B_i)\land \cR(\B_i,\C_i)$, $i=1,2$. Then, by monotonicity of T-norms, we get $\cR(\A_1,\C_1)\land \cR(\A_2,\C_2)\ge \cR(\A_1,\B_1)\land \cR(\B_1,\C_1) \land \cR(\A_2,\B_2)\land \cR(\B_2,\C_2),$ from which, by associativity and commutativity of $\land$ and the definition of $\cR$, we get $\cR((\A_1\,\A_2),(\C_1\,\C_2))\ge \cR((\A_1\,\A_2),(\B_1\,\B_2))\land \cR((\B_1\,\B_2),(\C_1\,\C_2))$ for any T-norm $\land$.
    \item Let $\A=\lambda x.\A'$, $\B=\lambda y. \B'$, and $\C=\lambda z. \C'$, where $x, y, z$ have the same type. By definition of $\cR$, without loss of generality, we can choose a fresh variable $u$ of the same type such that $\cR(\lambda x. \A',\lambda z. \C') = \cR(\A'\{x\mapsto u\}, \C'\{z\mapsto u\})$,  $\cR(\lambda x. \A',\lambda y.\B')=\cR(\A'\{x\mapsto u\}, \B'\{y\mapsto u\})$, and $\cR(\lambda y. \B', \allowbreak \lambda z. \C')=\cR(\B'\{y\mapsto u\}, \C'\{z\mapsto u\})$. By the induction hypothesis, we have $\cR(\A'\{x\mapsto u\}, \C'\{z\mapsto u\})\ge \cR(\A'\{x\mapsto u\}, \B'\{y\mapsto u\}) \land \cR(\B'\{y\mapsto u\}, \C'\{z\mapsto u\})$ for any T-norm, which implies $\cR(\lambda x. \A',\lambda z. \C')\ge \cR(\lambda x. \A',\lambda y.\B')\land \cR(\lambda y. \B',\lambda z. \C')$ for any T-norm.  \qed
\end{itemize}
\end{proof}

Given a similarity relation $\cR$, cut value $\cut$, and two terms $\A$, $\B$, we write $\A\simeq_{\cR,\cut}\B$ if $\cR(\A, \B) \ge \cut$. The number $\cR(\A, \B)$ is called the \emph{$\cR$-similarity degree} of $\A$ and $\B$.

We say $\sigma$ is $(\cR,\cut)$-\emph{more general} than $\vartheta$ with variable restriction $V$, written $\sigma \preceq_{\cR,\cut}^V \vartheta$, if there exists a substitution $\tau$ such that $x\sigma\tau \simeq_{\cR,\cut} x\vartheta$ for all $x\in V$.

\begin{definition}[Unification problem, unifier, unification degree, mgu]
  An $(\cR,\cut)$-\emph{unification equation} between terms $\A$ and $\B$ is written as $\A\simeq_{\cR,\cut}^?\B$, where $\cR$ is a similarity relation and $\cut$ is a cut value. 
An $(\cR,\cut)$-\emph{unification problem} is a finite set of $(\cR,\cut)$-uni\-fi\-cation equations.  A substitution $\sigma$ is a \emph{unifier} (or a solution) of an $(\cR,\cut)$-unification problem $\{\A_1\simeq_{\cR,\cut}^?\B_1,\ldots,\A_1\simeq_{\cR,\cut}^?\B_1\}$ with \emph{unification degree} $\deg$ if $\cR(\A_1\sigma,\B_1\sigma) \wedge \cdots \wedge \cR(\A_n\sigma, \allowbreak \B_n\sigma)  =\deg \ge \cut$.

Given an $(\cR,\cut)$-unification problem $P$, we say that its $(\cR,\cut)$-unifier $\sigma$ is a \emph{most general $(\cR,\cut)$-unifier} of $P$ if for any $(\cR,\cut)$-unifier $\vartheta$ of $P$ we have $\sigma \precsim_{\cR,\cut}^{\fv(P)} \vartheta$.
\end{definition}
   
We consider a special case of this problem: all terms that appear in the unification problems and unifiers should be higher-order patterns. In the rest of the paper, we make the following assumptions (those related to terms follow \cite{DBLP:conf/lics/Nipkow93}):

\begin{itemize}
    \item $\min$: T-norm $\wedge$ is the minimum T-norm,
    \item $\alpha$: $\alpha$-equivalent terms are identified,
    \item $\beta$: terms are $\beta$-normalized by default,
    \item $\eta$: terms are in $\eta$-expanded form except for the arguments of free variables, which are in $\eta$-normal form, i.e., bound variables.
\end{itemize}

\begin{example}
    \label{exmp:unif1}
    Higher-order pattern $(\cR,\cut)$-unification problems may have most general unifiers that are or are not degree-maximal, and may have degree-maximal unifiers that are or are not most general. We illustrate it with an $(\cR,\cut)$-unification problem between 
    \begin{align*}
        & \A=\lambda x. \lambda y. f(F(x), F(y)) \text{ and }\\
        & \B=\lambda x. \lambda y. g(a(G(y,x)), b(G(x,y))),
    \end{align*}
    where $\cR$ is defined as $\cR(f,g)=0.8$, $\cR(a,b)=0.6$, $\cR(b,c)=\cR(a,c)=0.5$, and the cut value is $\cut=0.4$. Consider the following $(\cR,\cut)$-unifiers of $\A$ and $\B$ together with their unification degrees:
    \begin{align*}
        \sigma_1 = {} & \{ F\mapsto \lambda x. a(H(x)), G\mapsto \lambda y.\lambda x. H(x)\}, & & \deg_1=0.6. \\
        \sigma_2 = {} & \{ F\mapsto \lambda x. b(H(x)), G\mapsto \lambda y.\lambda x. H(x)\}, & & \deg_2=0.6.\\
        \sigma_3 = {} & \{ F\mapsto \lambda x. c(H(x)), G\mapsto \lambda y.\lambda x. H(x)\}, & & \deg_3=0.5.\\
        \sigma_4 = {} & \{ F\mapsto \lambda x. a(a(x)), G\mapsto \lambda y.\lambda x. a(x)\}, & & \deg_4=0.6.\\
        \sigma_5 = {} & \{ F\mapsto \lambda x. c(a(x)), G\mapsto \lambda y.\lambda x. a(x)\}, & & \deg_5=0.5.
    \end{align*}
    From these unifiers, $\sigma_1,\sigma_2$ and $\sigma_3$ are most general ones on $V=\{F,G\}$. (They are equivalent to each other modulo similarity.) Hence, we have two most general unifiers $\sigma_1,\sigma_2$ with the maximum degree, and one with a lower degree. The unifier $\sigma_4$ is strictly less general on $V$ than the first three, yet its unification degree is maximal. As for $\sigma_5$, it is neither most general nor degree-maximal unifier.

    Our goal is to design an algorithm that computes a degree-maximal most general unifier. Hence, for this problem, computing either $\sigma_1$ or $\sigma_2$ would be the desired outcome.
\end{example}

\section{Similarity-based unification of higher-order patterns}
\label{sect:sim:unif}

In this section, we formulate a similarity-based unification algorithm in a rule-based manner. The rules operate on triples 
$ P;\, \sigma;\, \deg$, referred to as unification \emph{configurations}, where $P$ is a similarity-based unification problem, 
$\sigma$ is the substitution computed so far, and $\deg$ is the approximation degree, also computed so far. The similarity relation $\cR$ and the cut value $\cut$ are implicit parameters, and it is assumed that $\deg\ge \cut$ in configurations.

The rules are given below. The symbol $\uplus$ stands for disjoint union. The rule \textsf{LF} calls an auxiliary function  $\varelimsu$ (which is also defined below) with the sides of the selected equation as arguments. 
\infrule{Abs}{Abstraction}
 {\{ \lambda x. \A \simeq_{\cR,\cut}^? \lambda x. \B \} \uplus P;\, \sigma;\deg \leadsto \{ \A \simeq_{\cR,\cut}^? \B \} \cup P;\,\sigma;\deg.}

\infrule{Dec}{Decomposition}
 {\{ f(\A_1,\ldots,\A_n)  \simeq_{\cR,\cut}^? g(\B_1,\ldots,\B_n)\} \uplus P;\,\sigma;\deg \leadsto \\  \{ \A_1  \simeq_{\cR,\cut}^? \B_1,\ldots, \A_n  \simeq_{\cR,\cut}^? \B_n\} \cup P;\, \sigma; \deg \wedge \cR(f,g),}
 [\noindent where $n\ge0$ and $(\deg \wedge \cR(f,g)) \ge \cut$. ]

\infrule{SV}{Same variables}
 {\{ F(x_1,\ldots, x_n)  \simeq_{\cR,\cut}^? F(y_1,\ldots,y_n)\} \uplus P;\,\sigma;\,\deg \leadsto  P\vartheta;\,\sigma\vartheta;\, \deg,}
 \noindent where $n\ge 0$ and 
 \begin{itemize}
     \item $\{z_1,\ldots,z_m\}=\{x_i \mid \allowbreak x_i = y_i, \allowbreak 1\le i \le n\}$, $m\ge 0$ (i.e., each $z$ is a variable that appears in the same position in both the sequences $x_1,\ldots, x_n$ and $y_1,\ldots,y_n$),
     \item $\vartheta=\{F\mapsto \lambda x_1,\ldots,x_n. H(z_1,\ldots,z_m)\}$ for a fresh variable $H$ of the appropriate type.
 \end{itemize}

\infrule{Ori}{Orient}
 {\{ a(\B_1,\ldots,\B_m)  \simeq_{\cR,\cut}^? F(x_1,\ldots, x_n)\} \uplus P;\,\sigma;\,\deg \leadsto \\
 \{ F(x_1,\ldots, x_n)  \simeq_{\cR,\cut}^? a(\B_1,\ldots,\B_m)\} \cup P;\,\sigma;\deg,}
 [\noindent where $n,m\ge 0$, $F$ is a free variable and $a$ is a constant or $a\in \{x_1,\ldots,x_n\}$.] 

\infrule{LF}{Left-flex}
{ \{F(x_1,\ldots, x_n)  \simeq_{\cR,\cut}^? a(\B_1,\ldots,\B_m) \} \uplus P;\,\sigma;\,\deg \leadsto 
 P\vartheta;\,\sigma\vartheta;\, \deg
 }
 \noindent where $n,m\ge 0$ and
 \begin{itemize}
     \item $F\notin \fv(a(\B_1,\ldots,\B_m))$,
     \item $a$ is a constant, free variable, or $a\in \{x_1,\ldots,x_n\}$, 
     \item $\varelimsu(F(x_1,\ldots, x_n), a(\B_1,\ldots,\B_m))= \varphi$, 
     \item $\vartheta = \varphi|_{V}$, where $V=\{F\} \cup \fv(a(\B_1,\ldots,\B_m))$.
 \end{itemize}

\infrule{Fail}{Failure}
 { \{\A  \simeq_{\cR,\cut}^? \B \} \uplus P;\,\sigma;\,\deg \leadsto 
 \bot,
 }
[\noindent if no other rule applies to the selected equation $\A  \simeq_{\cR,\cut}^? \B$.]  \vspace{3mm}

For our minimum T-norm, it would suffice to have $\cR(f,g) \ge \cut$ instead of $(\deg \wedge \cR(f,g)) \ge \cut$ in the condition of the {\sf Dec} rule. However, we decided to keep this more general requirement since it can be used with any T-norm. See also Section~\ref{sect:discussion} for related discussion.

Analyzing the rules, it is not hard to see that when the failure rule applies, we have one of the following three cases:
\begin{itemize}
    \item $\A$ and $\B$ have different types,
    \item $\A=f(\A_1,\ldots,\A_n)$, $\B=g(\B_1,\ldots,\B_n)$, and $(\deg \wedge \cR(f,g)) < \cut$, or
    \item $\A=F(x_1,\ldots, x_n)$,  $\B= a(\B_1,\ldots,\B_m)$, and $F\in \fv(a(\B_1,\ldots,\B_m))$.
\end{itemize}

In each of these cases, $\A  \simeq_{\cR,\cut}^? \B$ is unsolvable, which justifies the name of the rule.

In the {\sf LF} rule, if $a$ is a free variable, then $s_1,\ldots,s_m$ are distinct bound variables (since we work only with higher-order patterns). By this rule, in order to unify a flexible term $\A$ and a term $\B$, the auxiliary function $\varelimsu(\A, \B)$ is called, which creates an initial configuration $\{ \A \simeq^? \B\};\varepsilon$ and applies the \textsf{VE1} and \textsf{VE2} rules below as long as possible. If the process stops with the configuration $\emptyset;\varphi$, we say that $\varphi$ is the answer computed by $\varelimsu$ for $\A$ and $\B$ and write $\varelimsu(\A, \B)=\varphi$. Note that $\varphi$ might contain bindings for variables introduced in the process of $\varelimsu$, which appear neither in $\A$ nor in $\B$. Therefore, in $\vartheta$, we keep only those bindings from $\varphi$ that are relevant to $\A$ and $\B$. Note also that $\varelimsu$ is independent of $\cR$ and does not involve the degree computation.

\infrule{VE1}{Variable elimination 1}
 { \{ F(x_1,\ldots, x_n)  \simeq^? a(\B_1,\ldots,\B_m) \} \uplus P;\sigma \leadsto \\
 \{  H_1(x_1,\ldots,x_n) \simeq^? \B_1,\ldots, H_m(x_1,\ldots,x_n) \simeq^? \B_m \} \cup P;\sigma\vartheta,}\vspace{2mm}
 \noindent where $n,m\ge 0$ and
 \begin{itemize}
     \item $a$ is a constant or $a\in \{x_1,\ldots,x_n\}$, 
     \item $\vartheta=\{F\mapsto \lambda x_1,\ldots,x_n. a(H_1(x_1,\ldots,x_n), \ldots, H_m(x_1,\ldots,x_n))\}$ where the variables $H_1,\ldots,H_m$ are fresh  and have appropriate types.
     \item (Note that if $s_i$ for $1\le i\le m$, is of function type, the term $H_i(x_1,\ldots,x_n)$ must be $\eta$-expanded in both the new configuration and the substitution.)
 \end{itemize}

\infrule{VE2}{Variable elimination 2}
 { \{F(x_1,\ldots, x_n)  \simeq^? G(y_1,\ldots,y_m)\} \uplus P;\sigma; \leadsto 
 P\vartheta;\sigma\vartheta,} \vspace{2mm}
 \noindent where $n,m\ge 0$, and
 \begin{itemize}
     \item $\{x_1,\ldots,x_n\}\cap \{y_1,\ldots,y_m\}=\{z_1,\ldots, z_k\}$, 
     \item $\vartheta=\{F\mapsto \lambda x_1,\ldots,x_n. H(z_1,\ldots,z_k), \allowbreak G\mapsto \allowbreak \lambda y_1,\ldots,y_m. H(z_1,\ldots,z_k)\}$ with $H$ being a fresh variable of the appropriate type.
 \end{itemize}

Note that when $\varelimsu$ is invoked, it is always called with terms $\A$ and $\B$ where $\A$ has the form $F(x_1,\ldots, x_n)$ such that the variable $F$ does not occur in $\B$. Applying \textsf{VE1} or \textsf{VE2} to the initial configuration $\{F(x_1,\ldots, x_n) \simeq^? \B \};\varepsilon$ creates a new configuration in which $F$ does not appear anymore, and if the variables introduced instead of it appear on the left-hand sides, then they are unique again. Hence, the configurations $P;\sigma$ to which \textsf{VE1} and \textsf{VE2} apply satisfy the following invariant:

\begin{itemize}
    \item If an equation of the form $F(x_1,\ldots, x_n) \simeq^? \B$ appears in $P$, then $F$ is unique.
\end{itemize}

Consequently, in \textsf{VE1}, it does not make sense to apply the generated substitution to the remaining unification problem. That is why we have $P$ and not $P\vartheta$ on the right-hand side of \textsf{VE1}. On the other hand, we need $P\vartheta$ in \textsf{VE2}, because $P$ might contain $G$.

Given a similarity relation $\cR$, the cut value $\cut$, and two terms $\A$ and $\B$, to find an ${(\cR,\cut)}$-unifier of $\A$ and $\B$, the algorithm {{\hopsu}} creates the initial configuration $\{\A \simeq_{\cR,\cut}^? \B\}; \varepsilon; 1$ and applies the rules above as long as possible. It means that {\hopsu} can be defined as the following strategy where \textsf{nf} stands for normal form:
\begin{align*}
   & {\hopsu} := \textsf{nf}(\stepsu). & & {\stepsu} := \textsf{choice(Abs, Dec, SV, Ori, LF, Fail).}
\end{align*}

It is easy to see that for each selected equation, only one rule applies. Hence, {\stepsu} is deterministic in the sense that it transforms a given configuration $C$ in only one way. We refer to the result of transformation as ${\stepsu}(C)$. 

\begin{example}
\label{exmp:unif}
We illustrate the algorithm to solve an $(\cR,\cut)$-unification problem from Example~\ref{exmp:unif1}.
We create the initial configuration \[\{\lambda x. \lambda y. f(F(x), F(y))\simeq_{\mathcal{R}, 0.4}^? \lambda x. \lambda y. g(a(G(y,x)), b(G(x,y)))\};\varepsilon;1\] and apply the transformation rules as follows: 
\begin{align*}
   & \{\lambda x. \lambda y. f(F(x), F(y))\simeq_{\mathcal{R}, 0.4}^? \lambda x. \lambda y. g(a(G(y,x)), b(G(x,y)))\};\varepsilon;1 \leadsto_{\textsf{Abs}}\\ &\{ \lambda y. f(F(x), F(y))\simeq_{\mathcal{R}, 0.4}^? \lambda y. g(a(G(y,x)), b(G(x,y)))\};\varepsilon;1 \leadsto_{\textsf{Abs}}\\ &\{ f(F(x), F(y))\simeq_{\mathcal{R}, 0.4}^?  g(a(G(y,x)), b(G(x,y)))\};\varepsilon;1 \leadsto_{\textsf{Dec}}\\
  &\{ F(x) \simeq_{\mathcal{R}, 0.4}^? a(G(y,x)), F(y) \simeq_{\mathcal{R}, 0.4}^? b(G(x,y))\} ;\varepsilon;0.8 \leadsto_{\textsf{LF}}\\
  & \text{Application of \textsf{LF} requires steps of the $\varelimsu$:} \\
   & \qquad \{ F(x) \simeq^? a(G(y,x))\};\varepsilon \leadsto_{\textsf{VE1}} \\ 
   &  \qquad \{H_1(x) \simeq^? G(y,x)\}; \{ F\mapsto \lambda x. a(H_1(x))\} \leadsto_{\textsf{VE2}}\\
   &  \qquad \emptyset; \{ F\mapsto \lambda x. a(H_2(x)), G\mapsto \lambda y\lambda x. H_2(x), H_1\mapsto \lambda x.H_2(x)\} \\
   & \text{$\varelimsu$ returns $\{ F\mapsto \lambda x. a(H_2(x)), G\mapsto \lambda y\lambda x. H_2(x)\}$ and we continue:}\\
   &\{ a(H_2(y)) \simeq_{\mathcal{R}, 0.4}^? b(H_2(y))\} ; \\
   & \qquad \{ F\mapsto \lambda x. a(H_2(x)), G\mapsto \lambda y\lambda x. H_2(x)\};0.8 \leadsto_{\textsf{Dec}}\\
   &\{ H_2(y) \simeq_{\mathcal{R}, 0.4}^? H_2(y)\} ;\{ F\mapsto \lambda x. a(H_2(x)), G\mapsto \lambda y\lambda x. H_2(x)\};0.6\leadsto_{\textsf{SV}} \\
   & \emptyset ;\{ F\mapsto \lambda x. a(H(x)), G\mapsto \lambda y\lambda x. H(x), H_2\mapsto \lambda x.H(x)\};0.6
\end{align*}
Hence, we obtained the substitution $\{ F\mapsto \lambda x. a(H(x)), G\mapsto \lambda y\lambda x. H(x), H_2\mapsto \lambda x.H(x)\}$ and the degree $0.6$. It is easy to check that the substitution is indeed a unifier of the given terms with degree $0.6$. Note also that not all the bindings from the substitution are relevant: we can restrict it only to the free variables of the given terms, getting $\{ F\mapsto \lambda x. a(H(x)), G\mapsto \lambda y\lambda x. H(x)\}$, which corresponds to $\sigma_1$ from Example~\ref{exmp:unif1}. If we applied the {\sf LF} rule to the second equation $F(y) \simeq_{\mathcal{R}, 0.4}^? b(G(x,y))$ instead of the first one, we would get the unifier $\sigma_2$ from Example~\ref{exmp:unif1}.

If we increase the cut value to, say, $0.7$, the problem does not have a solution and the algorithm reaches a configuration to which the failure rule applies:
\[\{ a(H_2(y)) \simeq_{\mathcal{R}, 0.7}^? b(H_2(y))\} ;\{ F\mapsto \lambda x. a(H_2(x)), G\mapsto \lambda y\lambda x. H_2(x)\};0.8 \leadsto_{\textsf{Fail}} \bot.\]

When $\cut=1$, the problem is a crisp (i.e., standard) higher-order pattern unification problem, which obviously is not solvable due to the mismatch between $f$ and $g$. The algorithm detects it, stopping with failure after the application of the abstraction steps.
\end{example}

\begin{theorem}[Termination of {\hopsu}]
    \label{thm:termination}
    {\hopsu} terminates for any input either with $\emptyset;\sigma; \deg$ or with $\bot$. 
\end{theorem}
\begin{proof}
We first establish that $\varelimsu(\A,\B)$ always terminates. This procedure starts by constructing the initial configuration $\{ \A \simeq_{\cR,\cut}^? \B\};\varepsilon$ and then repeatedly applies the \textsf{VE1} and \textsf{VE2} rules. To each configuration $P;\sigma$ we associate a complexity measure as the multiset of the sizes of equations in $P$, denoted by $\mathit{ms}(P):=\mset{\size(\A) + \size(\B) \mid \A \simeq_{\cR,\cut}^? \B \in P }.$  To order the measures, we use the multiset extension of the standard ordering on natural numbers (Dershowitz-Manna ordering~\cite{DBLP:journals/cacm/DershowitzM79}).
Both \textsf{VE1} and \textsf{VE2} rules strictly decrease this measure, ensuring that the reduction process cannot continue indefinitely. Consequently, $\varelimsu(\A, \B)$ must terminate.

Now we show termination of {\hopsu}. For this, we define a complexity measure $\cm(C)$ for a configuration $C=(P;\sigma;\deg)$, and show that $\cm(C)>\cm(\stepsu(C))$ holds. The measure is the triple $\cm(C)=\langle N_1, N_2, N_3 \rangle$ defined as follows:

\begin{itemize}
 \item $N_1$  is the number of distinct variables in $P$,
 \item $N_2 := \mathit{ms}(P)$,  
 \item $N_3$ is the number of equations in $P$ with a rigid left-hand side and a flexible right-hand side (rigid-flex equations).
\end{itemize}

We treat $\bot$ as a special configuration and define its measure as $\cm(\bot)=\langle 0,\emptyset,0\rangle$. Measures are compared lexicographically. For $N_2$, we use the Dershowitz-Manna ordering.
The table below shows which rule reduces which component of the measure. 
\begin{center}
 \begin{tabular}{|l|c|c|c|}
  \hline
  Rule &  $N_1$ & $N_2$ & $N_3$ \\
  \hline \hline
  {\sf LF} & $>$ & & \\
  {\sf Abs,\, Dec,\, SV, Fail } & $\ge$ & $>$ & \\
  {\sf Ori} & $\ge$ & $\ge$ & $>$ \\ \hline
 \end{tabular}
\end{center}
Hence, each rule strictly reduces $\cm(C)$. Since the ordering is well-founded, it implies the termination of the algorithm. It is obvious that there are only two alternatives for the final configuration: either the first component is empty, or it is $\bot$. It finishes the proof. \qed
\end{proof}

When {\hopsu} stops with $\emptyset;\sigma; \deg$, we say that {\hopsu} \emph{succeeds} with the \emph{computed answer}  $(\sigma,\deg)$ (or computes $(\sigma,\deg)$). If it stops with $\bot$, we say that {\hopsu} \emph{fails}. For a similarity relation $\cR$, a cut value $\cut$, and two terms $\A$ and $\B$, if {\hopsu} computes $(\sigma,\deg)$, we write ${\hopsu}(\A,\B, \cR,\cut)=(\sigma,\deg)$. 

\begin{lemma}
\label{lem:soundness:stepsu}
   Let $\cR$ be a similarity relation, $\cut$ a cut value and $P;\sigma;\deg$ a configuration. If ${\stepsu}(P;\sigma;\deg)=(P';\sigma\vartheta;\deg\wedge\deg')$ and $\tau$ is an $(\cR,\cut)$-unifier of $P'$ with degree $\deg''$, then $\vartheta\tau$ is an $(\cR,\cut)$-unifier of $P$ with degree  $\deg'\wedge\deg''$. 
\end{lemma}
\begin{proof}
First, we prove the soundness of $\varelimsu$. This will be needed to show the soundness of the \textsf{LF} rule.
Note that the rules \textsf{VE1} and \textsf{VE2} are the standard variable elimination rules used in the crisp higher-order pattern unification~\cite{DBLP:conf/lics/Nipkow93,DBLP:journals/logcom/Miller91} under the condition that the free variables in the left-hand side of the equations are unique in the problem. Then their soundness follows from the soundness of the corresponding rules in the standard higher-order pattern unification.

Further, we proceed by case distinction on the inference rules of \hopsu.
We illustrate here three cases when the selected rules are \textsf{Dec}, \textsf{SV}, and \textsf{LF}. For the other rules, the lemma can be shown similarly.

\begin{itemize}
    \item {\stepsu} transforms the given configuration by the \textsf{Dec} rule, i.e., we have
    \begin{align*}
       & \{ f(\A_1,\ldots,\A_n)  \simeq_{\cR,\cut}^? g(\B_1,\ldots,\B_n)\} \uplus P;\,\sigma;\,\deg \leadsto_{\textsf{Dec-SU}} \\
       & \qquad \{ \A_1  \simeq_{\cR,\cut}^? \B_1,\ldots, \A_n  \simeq_{\cR,\cut}^? \B_n\} \cup P;\, \sigma;\, \deg \wedge \cR(f,g). 
    \end{align*}
    Then $\vartheta=\varepsilon$ and $\deg'=\cR(f,g) \ge \cut$.
    Let $\tau$ be an $(\cR,\cut)$-unifier of the problem $\{ \A_1  \simeq_{\cR,\cut}^? \B_1,\allowbreak \ldots, \A_n  \simeq_{\cR,\cut}^? \B_n\} \cup P$  with the degree $\deg''$. This means, 
    \begin{align*}
        & \deg''=\cR(\A_1\tau, \B_1\tau) \land \cdots \land \cR(\A_n\tau, \B_n\tau)\land  d \ge \mu,
    \end{align*}
    where $d=\bigwedge_{(\A \simeq_{\cR,\cut}^? \B) \in P} \cR(\A\tau,\B\tau)$. Since $\deg'\ge \cut$ and $\deg''\ge \cut$, we have $\deg' \land \deg''\ge \cut$. Moreover, 
    \begin{align*}
         \deg' \land \deg'' = {} & \deg' \land \cR(\A_1\tau, \B_1\tau) \land \cdots \land \cR(\A_n\tau, \B_n\tau)\land  d \\
       = {} & \cR(f,g) \land \cR(\A_1\tau, \B_1\tau) \land \cdots \land \cR(\A_n\tau, \B_n\tau)\land  d \\
       = {} & \cR(f(\A_1,\ldots, \A_n)\tau, g(\B_1,\ldots, \B_n)\tau) \land d,
    \end{align*}
     which implies that $\tau = \vartheta\tau $ is an $(\cR,\cut)$-unifier of the unification problem $\{ f(\A_1,\ldots,\A_n)  \simeq_{\cR,\cut}^? g(\B_1,\ldots,\B_n)\} \uplus P$ with degree $\deg' \land \deg''$.

    \item {\stepsu} transforms the given configuration by the \textsf{SV} rule, i.e., we have
    \begin{align*}
       & \{ F(x_1,\ldots, x_n)  \simeq_{\cR,\cut}^? F(y_1,\ldots,y_n)\} \uplus P;\,\sigma;\,\deg \leadsto_{\textsf{SV}}  P\vartheta;\, \sigma\vartheta;\, \deg . 
    \end{align*}

   where $\vartheta=\{F\mapsto \lambda x_1,\ldots,x_n. H(z_1,\ldots,z_m)\}$,  $\{z_1,\ldots,z_m\}=\{x_i \mid  x_i = y_i,  1\le i \le n\}$. We also have $\deg'=1$. Let $\tau$ be an $(\cR,\cut)$-unifier of $P\vartheta$  with the degree $\deg''$. This means that $\deg''=\bigwedge_{(\A\vartheta \simeq_{\cR,\cut}^? \B\vartheta) \in P\vartheta}\cR(\A\vartheta\tau,\B\vartheta\tau)\ge\cut$. 
   Then $\vartheta\tau$ is also unifier of $\{ F(x_1,\ldots, x_n)  \simeq_{\cR,\cut}^? F(y_1,\ldots,y_n)\} \uplus P$ with degree $\deg''$. Indeed, we have
   \begin{align*}
       & \cR(F(x_1,\ldots, x_n)\vartheta\tau, F(y_1,\ldots,y_n)\vartheta\tau) \wedge \bigwedge_{(\A \simeq_{\cR,\cut}^? \B) \in P} \cR(\A\vartheta\tau,\B\vartheta\tau)  \\
     ={}  & \cR(H(z_1,\ldots, z_m)\tau, H(z_1,\ldots,z_m)\tau) \wedge \bigwedge_{(\A\vartheta \simeq_{\cR,\cut}^? \B\vartheta) \in P\vartheta} \cR(\A\vartheta\tau,\B\vartheta\tau) \\
     = {} & 1\wedge \deg'' = \deg''.
   \end{align*}

    \item {\stepsu} transforms the given configuration by the \textsf{LF} rule, i.e., we have
 \begin{align*}
       & \{ F(x_1,\ldots, x_n)  \simeq_{\cR,\cut}^? a(\B_1,\ldots,\B_m)\} \uplus P;\,\sigma;\,\deg \leadsto_{\textsf{LF}}  P\vartheta;\, \sigma\vartheta;\, \deg, 
    \end{align*}

where $\vartheta$ is an $(\cR,\cut)$-unifier of $F(x_1,\ldots, x_n)  \simeq_{\cR,\cut}^? a(\B_1,\ldots,\B_m)$ with degree $\deg'=1$. Let $\tau$ be an $(\cR,\cut)$-unifier of $P\vartheta$  with degree $\deg''$. Hence, $\deg''=\bigwedge_{(\A\vartheta \simeq_{\cR,\cut}^? \B\vartheta) \in P\vartheta}\cR(\A\vartheta\tau,\B\vartheta\tau)\ge\cut$. Then, since $\vartheta$ is an $(\cR,\cut)$-unifier of the problem $F(x_1,\ldots, x_n)  \simeq_{\cR,\cut}^? a(\B_1,\ldots,\B_m)$ with degree $\deg'$, so is $\vartheta\tau$. Hence, we get
   \begin{align*}
       & \cR(F(x_1,\ldots, x_n)\vartheta\tau, a(\B_1,\ldots,\B_m)\vartheta\tau) \wedge \bigwedge_{(\A \simeq_{\cR,\cut}^? \B) \in P} \cR(\A\vartheta\tau,\B\vartheta\tau)  \\
    ={}   & \deg' \wedge \bigwedge_{(\A\vartheta \simeq_{\cR,\cut}^? \B\vartheta) \in P\vartheta} \cR(\A\vartheta\tau,\B\vartheta\tau) = 1 \wedge \deg'' = \deg'',
   \end{align*}
which means that the substitution $\vartheta\tau$ is an $(\cR,\cut)$-unifier of the unification problem $ \{ F(x_1,\ldots, x_n)  \simeq_{\cR,\cut}^? a(\B_1,\ldots,\B_m)\} \uplus P$ with degree $\deg''$. \qed
\end{itemize}
\end{proof}

\begin{theorem}[Soundness of \hopsu]
\label{thm:soundness:hopsu}
    Given a similarity relation $\cR$, a cut value $\cut$, and two terms $\A$ and $\B$, if ${\hopsu}(\A,\B, \cR,\cut)=(\sigma,\deg)$, then $\deg\geq \cut$ and $\cR(\A\sigma,\B\sigma)=\deg$, i.e., $\sigma$ is an $(\cR,\cut)$-unifier of $\A$ and $\B$ with degree $\deg$.
\end{theorem}
\begin{proof}
    From ${\hopsu}(\A,\B, \cR,\cut)=(\sigma,\deg)$ we have the derivation
    \[P_0;\,\sigma_0;\,\deg_0  \leadsto P_1;\,\sigma_0\vartheta_1;\,\deg_0 \wedge \deg_1 \leadsto^* P_n;\, \sigma_0\vartheta_1\cdots\vartheta_n;\, \deg_0\wedge \deg_1 \wedge \cdots \wedge \deg_n,  \]
    where 
    \begin{align*}
        & P_0 = \{ \A \simeq_{\cR,\cut}^? \B\},\quad  \sigma_0=\varepsilon,\quad \deg_0=1,\\
        & P_n=\emptyset, \quad \sigma=\sigma_0\vartheta_1\cdots\vartheta_n = \vartheta_1\cdots\vartheta_n, \\  
        & \deg = \deg_0\wedge \deg_1 \wedge \cdots \wedge \deg_n=  \deg_1 \wedge \cdots \wedge \deg_n.
    \end{align*}
   Therefore, we get the desired result using induction on the length of the derivation with Lemma~\ref{lem:soundness:stepsu}. \qed
\end{proof}

\begin{lemma}
\label{lem:completeness:stepsu} Let $\cR$ be a similarity relation, $\cut$ a cut value and $P_0;\sigma_0;\deg_0$ a configuration. Assume $\tau$ is an $(\cR,\cut)$-unifier of $P_0$ with degree $\deg_\tau \ge \cut $ and $\sigma_0 \preceq_{\cR,\cut}^{\dom(\tau)} \tau$. 
Then we can make a step ${\stepsu}(P_0;\sigma_0;\deg_0)=(P_1;\sigma_1;\deg_1)$ such that 

\begin{itemize}
    \item $\deg_1=\deg_0\wedge \deg'_1$ for some $\deg'_1$,
    \item there exists a substitution $\varphi_1$ such that
      \begin{itemize}
          \item $\dom(\varphi_1)=\fv(P_1)\setminus \fv(P_0)$ (the set of new free variables in $P_1$),
          \item $\fv(\ran(\varphi_1)) \cap \dom(\tau)=\emptyset$,
          \item $\varphi_1\tau$ is an $(\cR,\cut)$-unifier of $P_1$ with degree $\deg_{\varphi_1\tau} \ge \cut$ such that $\deg_\tau \le \deg'_1 \wedge \deg_{\varphi_1\tau}$, and
          \item $\sigma_1 \precsim_{\cR,\cut}^{\dom(\varphi_1\tau)} \varphi_1\tau$.
      \end{itemize}
\end{itemize}
\end{lemma}
\begin{proof}
Assume without loss of generality that $\tau$ is idempotent. We prove the lemma by case distinction on the rules applied by {\stepsu}. For rules \textsf{Abs} and \textsf{Ori} it is simple. The \textsf{Fail} rule cannot apply, because it handles cases for which $P$ is not solvable, while by assumption, $\tau$ is a solution of $P$. We concentrate on the remaining rules.

\begin{description}
 
\item[\textsf{Dec:}] In this case, $P_0=\{ f(\A_1,\ldots,\A_n)  \simeq_{\cR,\cut}^? g(\B_1,\ldots,\B_n)\} \uplus P$ and the application of the rule gives $P_1; \sigma_1; \deg_1$, where $P_1=\{ \A_1  \simeq_{\cR,\cut}^? \B_1,\ldots, \A_n  \simeq_{\cR,\cut}^? \B_n\} \cup P$, $\sigma_1=\sigma_0$, and $\deg_1 = \deg_0 \wedge \cR(f,g)$. We take $\cR(f,g)$ as $\deg'_1$ and $\varphi_1=\varepsilon$. Assume $\tau$ is an $(\cR,\cut)$-unifier of $P_0$ with degree $\cut \le \deg_\tau $ and $\sigma_0 \preceq_{\cR,\cut}^{\dom(\tau)} \tau$. Then $\tau$ is also  an $(\cR,\cut)$-unifier of $P$ and we denote its degree by $\deg_P$. It implies that 
\begin{align*}
     \cut \le \deg_\tau = {} &\cR(f(\A_1\tau,\ldots,\A_n \tau),  g(\B_1\tau,\ldots,\B_n\tau)) \wedge \deg_P  \\
    ={}& \cR(f,g) \wedge \cR(\A_1\tau,\B_1\tau) \wedge \cdots \wedge \cR(\A_n\tau,\B_n\tau) \wedge \deg_P  \\
    ={}& \deg'_1 \wedge \cR(\A_1\varphi_1\tau,\B_1\varphi_1\tau) \wedge \cdots \wedge \cR(\A_n\varphi_1\tau,\B_n\varphi_1\tau) \wedge \deg_P.
\end{align*}
Denote $\cR(\A_1\varphi_1\tau,\B_1\varphi_1\tau) \wedge \cdots \wedge \cR(\A_n\varphi_1\tau,\B_n\varphi_1\tau) \wedge \deg_P$ by $\deg_{\varphi_1\tau}$. Then the previous inequality can be written as $\cut\le \deg'_1 \wedge \deg_{\varphi_1\tau}$, from which we get $\cut \le \deg_{\varphi_1\tau}$, implying that $\varphi_1\tau$ is an $(\cR,\cut)$-unifier of $\{ \A_1  \simeq_{\cR,\cut}^? \B_1,\ldots, \A_n  \simeq_{\cR,\cut}^? \B_n\} \cup P$ with degree $\deg_{\varphi_1\tau}$. Hence, we got that $\varphi_1\tau$ is an $(\cR,\cut)$-unifier of $P_1$ with degree $\cut \le \deg_{\varphi_1\tau}$ such that $\deg_\tau = \deg'_1 \wedge \deg_{\varphi_1\tau}$. Finally, from $\sigma_1=\sigma_0$, $\dom(\varphi_1\tau)=\dom(\tau)$, and $\sigma_0 \preceq_{\cR,\cut}^{\dom(\tau)} \tau$, we conclude $\sigma_1 \preceq_{\cR,\cut}^{\dom(\varphi_1\tau)} \tau$. 

   \item[\textsf{SV:}]  In this case, $P_0=\{ F(x_1,\ldots, x_n)  \simeq_{\cR,\cut}^? F(y_1,\ldots,y_n)\} \uplus P$ and the application of the rule gives $P_1; \sigma_1; \deg_1$ where $P_1= P\vartheta$, $\sigma_1=\sigma_0\vartheta$, and $\deg_1 = \deg_0$ for the substitution $\vartheta=\{F\mapsto \lambda x_1,\ldots,x_n. H(z_1,\ldots,z_m)\}$ where $\{z_1,\ldots,z_m\}=\{x_i \mid \allowbreak x_i = y_i, \allowbreak 1\le i \le n\}$, $m\ge 0$. 

Since $\tau$ is an $(\cR,\cut)$-unifier of $P_0$, we have $F\tau= \lambda x_1,\ldots,x_n. t$, where $t$ is such a term that if some $x_i$ or $y_i$, $1 \le i \le n$, occurs in $t$, then this variable should belong to $\{z_1,\ldots,z_m\}$, otherwise $\tau$ would not be a unifier of $F(x_1,\ldots, x_n) \simeq_{\cR,\cut}^? F(y_1,\ldots,y_n)$.  It implies that $F\tau$ can be actually represented as
\begin{equation}
    \label{eq:sv-su-:F}
    F\tau= (\lambda x_1,\ldots,x_n. H(z_1,\ldots,z_m))\{ H\mapsto \lambda z_1,\ldots,z_m. t\}.
\end{equation}

We take $\deg'_1=1$ and $\varphi_1 = \{H\mapsto \lambda z_1,\ldots,z_m. t \}$. Then

\begin{itemize}
    \item $\deg_1=\deg_0 = \deg_0\wedge \deg'_1$,
    \item $\dom(\varphi_1)= \{H\} = \fv(P_1)\setminus \fv(P_0)$,
    \item $\fv(\ran(\varphi_1)) \cap \dom(\tau)=\emptyset$, since $\tau$ is idempotent.
\end{itemize}

What remains to show is 

\begin{itemize}
    \item  $\varphi_1\tau$ is an $(\cR,\cut)$-unifier of $P_1$ with degree $\cut \le \deg_{\varphi_1\tau}$ such that $\deg_\tau \le \deg'_1 \wedge \deg_{\varphi_1\tau}=\deg_{\varphi_1\tau}$, and
    \item $\sigma_1 \precsim_{\cR,\cut}^{\dom(\varphi_1\tau)} \varphi_1\tau$.
\end{itemize}

To show the first of them, we prove that $\vartheta\varphi_1\tau$ is an $(\cR,\cut)$-unifier of $P_0$ with degree $\deg_\tau$. For this, take a variable $X\in \fv(P_0)$. If $X\neq F$, then $X\vartheta\varphi_1\tau = X\tau$ since $X\vartheta\varphi_1=X$. As for $X=F$, we have 
\begin{align*}
    F\vartheta\varphi_1\tau= {}  & \big( \lambda x_1,\ldots,x_n. H(z_1,\ldots,z_m) \big)\varphi_1\tau  = \text{ (by equality (\ref{eq:sv-su-:F}))}\\
    {}= {} & F\tau\tau =\text{ (by the idempotence of $\tau$) } = F\tau.
\end{align*}

Hence, we can conclude that $\vartheta\varphi_1\tau$ is an $(\cR,\cut)$-unifier of $P_0$ with degree $\cut \le\deg_{\vartheta\varphi_1\tau}=\deg_\tau$, implying that $\varphi_1\tau$ is an $(\cR,\cut)$-unifier of $P_1=P\vartheta$ with degree $\deg_\tau$.

Now we should prove $\sigma_1 \precsim_{\cR,\cut}^{\dom(\varphi_1\tau)} \varphi_1\tau$. First, note that from $\sigma_0 \precsim_{\cR,\cut}^{\dom(\tau)} \tau$, there exists a substitution $\psi_0$ (whose domain is disjoint from the free variables in the range of $\tau$)  such that $X\sigma_0\psi_0 \simeq_{\cR,\cut} X\tau$ for all $X\in \dom(\tau)$. Define $\psi_1=\varphi_1 \cup \psi_0$. (Note that $\psi_0$ can be chosen so that $\dom(\psi_0)\cap (\dom(\varphi_1) \cup \fv(\ran(\varphi_1)))=\emptyset$ and $\dom(\varphi_1)\cap (\dom(\psi_0)\cup \fv(\ran(\psi_0))) = \emptyset$.)  Take a variable $X \in \dom(\varphi_1\tau)=\{H\} \uplus \dom(\tau)$ and show that $X\sigma_1 \psi_1 \simeq_{\cR,\cut} X\varphi_1\tau$.

\begin{enumerate}
    \item $X=H$. Then $X\sigma_1 \psi_1 = X\psi_1=X\varphi_1 = \lambda z_1,\ldots,z_m.t$, where $t$ is such a term that $F\tau = \lambda x_1,\ldots,x_n.t$. But then, $\fv(t)\cap \dom(\tau) =\emptyset$ and we have $\lambda z_1,\ldots,z_m.t=\lambda z_1,\ldots,z_m.t\tau = X\varphi_1\tau$. Hence, $X\sigma_1 \psi_1 =X\varphi_1\tau$.
    \item $X=F$. Then 
    \begin{align*}
         X\sigma_1 \psi_1 ={} & X\sigma_0\vartheta \psi_1 = X\vartheta\psi_1 = \lambda x_1,\ldots,x_n.H(z_1,\ldots,z_m)\psi_1\\
         ={} &  \lambda x_1,\ldots,x_n.H(z_1,\ldots,z_m)\varphi_1 = \lambda x_1,\ldots,x_n.t.\\
         X\varphi_1\tau = {} & X\tau = \lambda x_1,\ldots,x_n.t.
    \end{align*}
     Hence, also in this case, $X\sigma_1 \psi_1 =X\varphi_1\tau$.
     \item $X\in \dom(\tau)\setminus \{F\}$. We have two subcases:
     \begin{itemize}
         \item $X\sigma_0\vartheta=X\sigma_0$. Then 
         \begin{align*}
             X\sigma_1 \psi_1= {} & X\sigma_0\vartheta\psi_1=X\sigma_0\psi_1  \\
            = {} &  \simeq_{\cR,\cut} X\tau \text{ (by the assumption on $\psi_0$) }\\
             ={} & X\varphi_1\tau \text{ (since $X\notin \dom(\varphi_1$)).}
         \end{align*}
         \item $X\sigma_0\vartheta \neq X\sigma_0$, then we should show $Z\sigma_1\psi_1 \simeq_{\cR,\cut} Z\varphi_1\tau$ for all $Z\in \fv(X\sigma_0\vartheta)$, which implies $X\sigma_1 \psi_1 \simeq_{\cR,\cut} X\varphi_1\tau$.

         Proving $Z\sigma_1\psi_1 \simeq_{\cR,\cut} Z\varphi_1\tau$ for all $Z\in \fv(X\sigma_0\vartheta)$. From $X\sigma_0\vartheta \neq X\sigma_0$ we have $F\in \fv(X\sigma_0)$ and, therefore, $H\in \fv (X\sigma_0\vartheta) $. However, for $H$ we already saw that $H\sigma_1\psi_1 = H\varphi_1\tau$. For any other $Y\in \fv(X\sigma_0\vartheta)$ (i.e., $Y\neq H$, $Y\neq F$) we should have  $Y\in \fv(X\sigma_0)$, implying $Y\notin \dom(\sigma_0)$ (since $Y$ appears in the range of $\sigma_0$ and $\sigma$'s are idempotent). From $Y\notin \dom(\sigma_0)$ and $Y\neq F$ we get $Y\notin \dom(\sigma_1)$ and, consequently, $Y\sigma_1\psi_1=Y\psi_1$. Then we have
         \begin{align*}
             Y\sigma_1\psi_1= {} & Y\psi_1 = \text{ (since $Y\neq H$) } Y\psi_0 = \text{ (since $Y\notin \dom(\sigma_0)$) } \\
             ={} & Y\sigma_0\psi_0 \simeq_{\cR,\cut} Y\tau = Y\varphi_1\tau.
         \end{align*}
         It proves $Z\sigma_1\psi_1 \simeq_{\cR,\cut} Z\varphi_1\tau$ for all $Z\in \fv(X\sigma_0\vartheta)$.
       \end{itemize}
       In both cases, we get $X\sigma_1 \psi_1 \simeq_{\cR,\cut} X\varphi_1\tau$.
     \end{enumerate}
   Hence, in all three cases, we proved  $X\sigma_1 \psi_1 \simeq_{\cR,\cut} X\varphi_1\tau$. It implies that $\sigma_1 \precsim_{\cR,\cut}^{\dom(\varphi_1\tau)} \varphi_1\tau$.

\item[\textsf{LF:}] Here $P_0=\{F(x_1,\ldots, x_n)  \simeq_{\cR,\cut}^? a(\B_1,\ldots,\B_m) \} \uplus P$ and the application of the rule gives $P_1; \sigma_1; \deg_1$ where $P_1= P\vartheta$, $\sigma_1=\sigma_0\vartheta$, and $\deg_1 = \deg_0$ for the substitution $\vartheta$ where  $\vartheta$ is computed by $\varelimsu(F(x_1,\ldots, x_n), a(\B_1,\ldots,\B_m))$. Note that $\varelimsu$ applies (crisp) rules \textsf{VE1} and \textsf{VE2}. From~\cite{DBLP:conf/lics/Nipkow93} it follows that repeated application of them leads to a most general pattern unifier of $F(x_1,\ldots, x_n)$ and $a(\B_1,\ldots,\B_m)$. Therefore, $\vartheta$ is such an mgu and we have $F(x_1,\ldots, x_n)\vartheta = a(\B_1,\ldots,\B_m)\vartheta$ and $\vartheta \precsim_{\cR,\cut}^{\dom(\vartheta)}\tau$, where $\dom(\vartheta)=\{F\}\cup \fv(a(s_1,\ldots,s_m))$. Note that all variables that appear in the range of $\vartheta$ are new because none of the rules retain the original variables from $a(\B_1,\ldots,\B_m)$. Then there exists a substitution $\varphi_1$ such that 
\begin{itemize}
    \item $\dom(\varphi_1)= \fv(P_1) \setminus \fv(P_0)= \fv(\ran(\vartheta)) $,  
    \item $\fv(\ran(\varphi_1))\cap \dom(\tau)=\emptyset$, and
    \item $X\vartheta\varphi_1 \simeq_{\cR,\cut} X\tau$ for all $X\in \dom(\vartheta)$.
\end{itemize}
Therefore, by the idempotence of $\tau$, we get $X \vartheta\varphi_1\tau \simeq_{\cR,\cut} X\tau$ for all $X\in \fv(P_0)$.
Moreover, we have
\begin{itemize}
    \item $\cR(F \vartheta\varphi_1\tau,  a(\B_1,\ldots,\B_m)\vartheta\varphi_1\tau ) \ge \cR(F \tau,  a(\B_1,\ldots,\B_m) \tau )$, and
    \item $X\vartheta\varphi_1\tau = X\tau$ for all $X\in \fv(P_0)\setminus \dom(\vartheta)$. 
\end{itemize}

Therefore, $\vartheta\varphi_1\tau$ solves $P$ with the degree $\deg_{\vartheta\varphi_1\tau}$ that is better than $\deg_\tau$. Note that $\deg_{\vartheta\varphi_1\tau}$  and the unification degree of $\varphi_1\tau$ for $P\vartheta$ (i.e., $\deg_{\varphi_1\tau}$) are the same since $s\simeq_{\cR,\cut}^? t\in P$ iff $s\vartheta\simeq_{\cR,\cut}^? t\vartheta\in P\vartheta$ and we have
\begin{align*}
     \deg_\tau\le \deg_{\vartheta\varphi_1\tau} = {} & \bigwedge_{s\simeq_{\cR,\cut}^? t\in P} \cR(s\vartheta\varphi_1\tau, t\vartheta\varphi_1\tau) \\
    ={} & \bigwedge_{s\vartheta\simeq_{\cR,\cut}^? t\vartheta\in P\vartheta} \cR(s\vartheta\varphi_1\tau, t\vartheta\varphi_1\tau) = \deg_{\varphi_1\tau}.
\end{align*}

From $\deg_1 = \deg_0 = \deg'_1 \wedge \deg_0$ for $\deg'_1=1$, this implies that $\varphi_1\tau$ solves $P_1=P\vartheta$ with degree $\deg_\tau\le \deg_{\varphi_1\tau}$.

The last thing to prove is $\sigma_1 \precsim_{\cR,\cut}^{\dom(\varphi_1\tau)} \varphi_1\tau$. From $\sigma_0 \precsim_{\cR,\cut}^{\dom(\tau)} \tau$, there exists a substitution $\psi_0$ such that
\begin{itemize}
    \item $X\sigma_0\psi_0 \simeq_{\cR,\cut} X\tau$ for all $X\in \dom(\tau)$,
    \item $\dom(\psi_0) \cap \fv(\ran(\tau))=\emptyset$,
    \item $\dom(\psi_0)\cap (\dom(\varphi_1) \cup \fv(\ran(\varphi_1)))=\emptyset$, and 
    \item $\dom(\varphi_1)\cap (\dom(\psi_0)\cup \fv(\ran(\psi_0))) = \emptyset$.
\end{itemize}

Let $\psi_1=\varphi_1\cup \psi_0$. To prove $\sigma_1 \precsim_{\cR,\cut}^{\dom(\varphi_1\tau)} \varphi_1\tau$, we show $X\sigma_1 \psi_1 \simeq_{\cR,\cut} X\varphi_1\tau$ for all $X\in \dom(\varphi_1\tau)$. We have the following cases:
\begin{itemize}
    \item $X\in\dom(\varphi_1)$. Then $X\sigma_1\psi_1 =X\sigma_0\vartheta \psi_1 = X\vartheta \psi_1=X \psi_1=X\varphi_1 = X\varphi_1\tau$ (since $\fv(\ran(\varphi_1)\cap \dom(\tau)=\emptyset$). 
    \item $X \in \dom(\vartheta)$. Then $X\sigma_1\psi_1=X\sigma_0\vartheta \psi_1 = X \vartheta \psi_1 = X\vartheta \varphi_1 \simeq_{\cR,\cut} X\tau = X\varphi_1\tau$.
    \item If $X\in \dom(\tau) \setminus \dom(\vartheta)$.  We have two subcases:
    \begin{itemize}
        \item If $X\sigma_0\vartheta =X\sigma_0$. then $X\sigma_1\psi_1=X\sigma_0\vartheta \psi_1 = X\sigma_0 \psi_1 =X\sigma_0 \psi_0 \simeq_{\cR,\cut} x\tau = x\varphi_1\tau$. 
        
        \item If $X\sigma_0\vartheta \neq X\sigma_0$, then we should show $Z\sigma_1\psi_1 \simeq_{\cR,\cut} Z\varphi_1\tau$ for all $Z\in \fv(X\sigma_0\vartheta)$, which implies $X\sigma_1 \psi_1 \simeq_{\cR,\cut} X\varphi_1\tau$.

         Proving $Z\sigma_1\psi_1 \simeq_{\cR,\cut} Z\varphi_1\tau$ for all $Z\in \fv(X\sigma_0\vartheta)$. For each $Y\in \fv(\ran(\vartheta))= \dom(\varphi_1)$ we already saw that $Y\sigma_1\psi_1  = Y\varphi_1\tau$. For each $Y\in \fv(\ran(\sigma_0))\setminus (\dom(\vartheta))$ we have $Y\sigma_0\psi_0 \simeq_{\cR,\cut} Y\tau = Y\varphi_1\tau$. On the other hand, for such a $Y$ we have $Y\sigma_0\psi_0 = Y\psi_0 = Y = Y\psi_1 = Y\sigma_1\psi_1$. Hance, in both cases we got $Y\sigma_1\psi_1 \simeq_{\cR,\cut} Y\varphi_1\tau$, which proves that $Z\sigma_1\psi_1 \simeq_{\cR,\cut} Z\varphi_1\tau$ for all $Z\in \fv(X\sigma_0\vartheta)$.
    \end{itemize}
\end{itemize}
Hence, in all cases, we get $X\sigma_1\psi_1 \simeq_{\cR,\cut} X\varphi_1\tau$. It implies that $\sigma_1\precsim_{\cR,\cut}^{\dom(\varphi_1\tau)} \varphi_1\tau$.  \qed
\end{description}    
\end{proof}

\begin{theorem}[Completeness of \hopsu]
\label{thm:completeness:hopsu}
    Given a similarity relation $\cR$, a cut value $\cut$, and two terms $\A$ and $\B$, if there exists an $(\cR,\cut)$-unifier $\tau$ of $\A$ and $\B$ with degree $\deg_\tau \ge \cut$, then  $\hopsu(\A,\B,\cR,\cut)$ computes an answer $(\sigma,\deg)$ such that $\sigma \precsim_{\cR,\cut}^{\dom(\tau)} \tau$ and $\deg \ge \deg_\tau$. 
\end{theorem}
\begin{proof}
First note that for any configuration $P;\sigma;\deg$, if $P$ is $(\cR,\cut)$-unifiable then ${\stepsu}(P;\sigma;\deg)\neq \bot$. This is based on the observation that {\stepsu} gives $\bot$ in only three cases: if $P$ contains an equation between terms of different types, or an equation $f(\A_1,\ldots,\A_n) \simeq_{\cR,\cut}^? g(\B_1,\ldots,\B_n)$ with $\deg \land \cR(f,g)<\cut$, or an equation $F(\A_1,\ldots,\A_n) \simeq_{\cR,\cut}^? a(\B_1,\ldots,\B_n)$ with $F\in \fv (a(\B_1,\ldots,\B_n))$. But in none of these cases is $P$ $(\cR,\cut)$-unifiable.

Hence, for the initial configuration $P_0;\sigma_0;\deg_0=\{\A \simeq_{\cR,\cut}^? \B\};\varepsilon;1$ we have that $\tau$ is an $(\cR,\cut)$-unifier of $P_0$ with degree $\deg_\tau \ge \cut $ and $\sigma_0 \precsim_{\cR,\cut}^{\dom(\tau)} \tau$. By Lemma~\ref{lem:completeness:stepsu}, we can make ${\stepsu}(P_0;\sigma_0;\deg_0)=(P_1;\sigma_1;\deg_1)$ such that

\begin{itemize}
    \item $\deg_1=\deg_0\wedge \deg'_1$ for some $\deg'_1$,
    \item there exists a substitution $\varphi_1$ such that
      \begin{itemize}
          \item $\dom(\varphi_1)=\fv(P_1)\setminus \fv(P_0)$ (the set of new free variables in $P_1$),
          \item $\fv(\ran(\varphi_1)) \cap \dom(\tau)=\emptyset$,
          \item $\varphi_1\tau$ is an $(\cR,\cut)$-unifier of $P_1$ with degree $\deg'_\tau \ge \cut $ such that $\deg_\tau \le \deg'_1 \wedge \deg'_\tau$, and
          \item $\sigma_1 \precsim_{\cR,\cut}^{\dom(\varphi_1\tau)} \varphi_1\tau$.
      \end{itemize}
\end{itemize}
The obtained configuration $P_1;\sigma_1;\deg_1$ and $\varphi_1\tau$ satisfy the conditions of Lemma~\ref{lem:completeness:stepsu}. Hence, we can iterate applications of {\stepsu} (finitely many times, due to Theorem~\ref{thm:termination}), obtaining the maximal chain of configurations $P_1;\sigma_1;\deg_1 \leadsto \cdots \leadsto P_n;\sigma_n;\deg_n$. Then 

\begin{itemize}
    \item $P_n = \emptyset$, otherwise it would contradict the fact that it has a solution;
    \item $\sigma_n\precsim_{\cR,\cut}^{\dom(\varphi_1\cdots\varphi_n\tau)} \varphi_1\cdots\varphi_n\tau$;
    \item $\deg_n=\deg_0 \wedge \deg'_1 \wedge \cdots \wedge \deg'_n$, where $\deg_0=1$ and for all $1\le i\le n$, $\deg'_i$ is such that $\deg_i=\deg_{i-1}\wedge \deg'_i$.
\end{itemize}

Hence, we obtained $\hopsu(\A,\B,\cR,\cut)=(\sigma_n,\deg_n)$, where $\sigma_n\precsim_{\cR,\cut}^{\dom(\varphi_1\cdots\varphi_n\tau)} \varphi_1\cdots\varphi_n\tau$, and there exists $\deg'_\tau$ such that $\deg_\tau\le \deg'_1 \wedge \cdots \wedge \deg'_n\wedge \deg'_\tau = \deg_n \wedge \deg'_\tau \le \deg_n$. Since  $\dom(\varphi_1\cdots\varphi_n) \cap \dom(\tau) = \emptyset$ and $\fv(\ran(\varphi_1\cdots\varphi_n)) \cap \dom(\tau) = \emptyset$, the composition $\varphi_1\cdots\varphi_n\tau$ can be also represented as the disjoint union of the set representations of these substitutions: $(\varphi_1\cdots\varphi_n) \uplus \tau$. Therefore, from $\sigma_n\precsim_{\cR,\cut}^{\dom(\varphi_1\cdots\varphi_n\tau)} \varphi_1\cdots\varphi_n\tau$ we get $\sigma_n\precsim_{\cR,\cut}^{\dom(\tau)} \tau$.
We can take $\sigma_n$ and $\deg_n$ in the role of $\sigma$ and $\deg$, respectively, which finishes the proof. \qed
\end{proof}

\section{Discussion}
    \label{sect:discussion}

\subsubsection*{From fuzzy to crisp.}
    When $\cut=1$, we are essentially in the crisp case, and our algorithm can be seen as a rule-based variant of the higher-order pattern unification algorithm~\cite{DBLP:journals/logcom/Miller91,DBLP:conf/lics/Nipkow93}. It is closer to Nipkow's version~\cite{DBLP:conf/lics/Nipkow93} as our use of $\varelimsu$ introduces a form of configuration transformation strategy. On the other hand, our approach is more flexible in selecting which equations to transform, since our rules operate on sets rather than lists. In \cite{DBLP:conf/lics/Nipkow93}, it was noted that using sets instead of lists would lead to divergence, but with our rules and control, it does not happen. We imitate Nipkow's requirement to work at the head of the list, but only when the left-hand side of the selected equation is a flexible term. Below is an example demonstrating how our algorithm handles problematic equations $\{F =^? c(G), G=^? c(F)\}$ from \cite{DBLP:conf/lics/Nipkow93}:
    \begin{align*}
        & \{F \simeq_{\cR,1}^? c(G),\ G \simeq_{\cR,1}^? c(F)\}; \varepsilon; 1 \leadsto_{\textsf{LF}} \\
        & \text{Calling } \varelimsu: \\
        & \qquad \{F \simeq^? c(G)\}; \varepsilon \leadsto_{\textsf{VE1}}\\
        & \qquad  \{H \simeq^? G\}; \{F \mapsto c(H)\} \leadsto_{\textsf{VE2}} \\
        & \qquad \emptyset; \{F \mapsto c(G), \, H\mapsto G\} \\
        & \{G \simeq_{\cR,1}^? c(c(G))\}; \{F \mapsto c(G)\}; 1 \leadsto_{\textsf{Fail}} \\
        & \bot.
    \end{align*}
    Hence, {\hopsu} fails, as expected.

    It should also be mentioned that our rules allow a rather simple termination measure that facilitates the direct termination proof (instead of translating the problem into first-order unification). 
\subsubsection*{T-norms.}
    Note that using the minimum T-norm (in fact, its idempotence property) is important for the completeness of the algorithm as well as for computing unifiers with the maximum approximation degree. (This is not specific to the higher-order case: it holds even in the first-order case.) The condition $(\deg \wedge \cR(f,g))>\cut$ in the decomposition rule (which is the rule actually computing degrees) is general and applies to any T-norm. 
    For nonidempotent T-norms, we should treat unification problems as multisets. However, even under this modification, the algorithm would be incomplete. To illustrate this, consider a simple similarity relation $\cR(a,b)=0.5$ and the product T-norm defined as $s_1 \wedge s_2 = s_1 * s_2$. Let $\A=f(X,X,X)$, $\B=f(a,b,b)$, and the cut value $\cut=0.3$. Then the unification problem $\A \simeq_{\cR,\cut}^? \B$ has a unifier $\{X \mapsto b\}$, because $\cR(f(b,b,b), f(a,b,b))=0.5>0.3$. However, the (modified) algorithm fails to compute it:
    \begin{align*}
        & \llbrace f(X,X,X) \simeq_{\cR,0.3}^? f(a,b,b)\rrbrace; \varepsilon; 1 \leadsto_{\sf Dec} \\
        & \llbrace X \simeq_{\cR,0.3}^? a, X \simeq_{\cR,0.3}^? b, X \simeq_{\cR,0.3}^? b\rrbrace; \varepsilon; 1 \leadsto_{\sf LF} \\
        & \llbrace a \simeq_{\cR,0.3}^? b, a \simeq_{\cR,0.3}^? b\rrbrace; \{X\mapsto a\}; 1 \leadsto_{\sf Dec} \\
        & \llbrace a \simeq_{\cR,0.3}^? b\rrbrace; \{X\mapsto a\}; 0.5 \leadsto_{\textsf{Fail}} \\
        & \bot
    \end{align*}
    The computation stops with failure
    because the decomposition attempt of $a$ and $b$ would involve the check $0.5 * \cR(a,b)\ge 0.3$, which fails since $0.5 * \cR(a,b) = 0.5 * 0.5 = 0.25 < 0.3$.

    For a lower cut value, e.g., $\cut=0.1$, the algorithm would compute a solution $\{X\mapsto a\}$ with degree $0.25$, failing to discover that there exists another unifier, $\{X\mapsto b\}$, with a better degree $0.3$.

    However, if we consider only linear unification problems (i.e., those where no free variable occurs more than once), the modified algorithm is complete and computes a unifier with the highest degree even for nonidempotent T-norms.

    It is possible to regain completeness for nonidempotent T-norms for the general case by modifying the subprocedure $\varelimsu$: In the \textsf{VE1} rule, allow the variable $F$ to be replaced not only by the term \[\lambda x_1,\ldots,x_n. a(H_1(x_1,\ldots,x_n), \ldots, H_m(x_1,\ldots,x_n)),\] but by any term (non-deterministically) \[\lambda x_1,\ldots,x_n. b(H_1(x_1,\ldots,x_n), \ldots, H_m(x_1,\ldots,x_n))\]  where $b$ is similar to $a$ (provided that it still keeps the degree computed so far above the cut value). The computed degree should be also changed accordingly. Such a modification would lead to multiple (finitely many) most general unifiers, out of which one would need to keep only those that have the highest unification degrees. A detailed discussion of this algorithm goes beyond the scope of this paper, though.

\section{Conclusion}
\label{sect:concl}
We studied unification of higher-order patterns modulo fuzzy similarity relations, using the minimum T-norm. This problem, like its classical (crisp) counterpart, is unitary. We developed a rule-based unification algorithm and proved its termination, soundness, and completeness. The computed most general unifier has the best approximation degree. Our work extends, on one hand, the well-known higher-order pattern unification~\cite{DBLP:journals/logcom/Miller91,DBLP:conf/lics/Nipkow93} to accommodate fuzzy similarity relations. On the other hand, it generalizes the weak unification algorithm~\cite{DBLP:journals/tcs/Sessa02} from first-order terms to simply-typed lambda terms.

There are some interesting directions for future work, which involve both practical applications and further theoretical developments:

\begin{itemize}
    \item Incorporating the unification algorithm developed in this paper into a higher-order fuzzy logic programming formalism combining the powers of languages like Lambda-Prolog~\cite{DBLP:books/daglib/0036008} and FASILL~\cite{DBLP:journals/ijar/IranzoMR20}. It can serve as a foundation of a flexible higher-order knowledge-based system.
    \item Relaxing similarity relations to adapt to more diverse scenarios. One possibility is, e.g., lifting the transitivity requirement and considering proximity relations. This can lead to the generalization of block-based~\cite{DBLP:journals/fss/IranzoR15} or class-based~\cite{DBLP:conf/fuzzIEEE/PauK21,DBLP:conf/lopstr/KutsiaP19} approximate unification algorithms to the higher-order setting. Another option is defining approximation based on (Lawverean) quantales like in quantitative algebras and generalizing the results from \cite{DBLP:conf/ijcar/EhlingK24} to a higher-order case. 
\end{itemize}
\subsubsection{\ackname} This work was supported by the Austrian Science Fund (FWF) project P~35530, and NATO Science for Peace and Security (SPS) Programme G6133.
%


\end{document}